\documentclass[11pt]{article}

\usepackage[T1]{fontenc}
\usepackage{amsmath,amssymb,amsthm,mathtools}
\usepackage{booktabs}
\usepackage{enumitem}
\usepackage{geometry}
\usepackage{graphicx}
\usepackage{microtype}
\usepackage[numbers,sort&compress]{natbib}
\usepackage{xcolor}
\usepackage[hidelinks]{hyperref}

\setlist{nosep,leftmargin=*}
\allowdisplaybreaks
\newtheorem{theorem}{Theorem}[section]
\newtheorem{proposition}[theorem]{Proposition}
\newtheorem{lemma}[theorem]{Lemma}
\newtheorem{corollary}[theorem]{Corollary}
\theoremstyle{definition}
\newtheorem{assumption}[theorem]{Assumption}
\newtheorem{definition}[theorem]{Definition}
\theoremstyle{remark}
\newtheorem{remark}[theorem]{Remark}

\newcommand{\E}{\mathbb{E}}
\newcommand{\R}{\mathbb{R}}
\newcommand{\1}{\mathbf{1}}
\newcommand{\dd}{\,\mathrm{d}}
\newcommand{\argmax}{\operatorname*{arg\,max}}
\newcommand{\clip}{\operatorname{clip}}

\hypersetup{
  pdftitle={When Does the Best Sampling Temperature Rise with the Budget? Sufficient Conditions for Pass@k},
  pdfauthor={Changsu Jeong},
  pdfsubject={Sufficient conditions for budget-dependent optimal sampling temperature under Pass@k},
  pdfkeywords={Pass@k, inference-time scaling, sampling temperature, monotone comparative statics}
}

\title{\textbf{When Does the Best Sampling Temperature Rise with the Budget?}\\
Sufficient Conditions for \(\pass@k\)}
\author{Changsu Jeong\\
\small Independent Researcher}
\date{July 2026}

\begin{document}
\maketitle

\begin{abstract}
The temperature that maximizes pass@\(k\) is often low for a small sampling
budget and higher for a large budget. This pattern has been reported from
Codex through recent multi-sample inference studies. It is not an algebraic
property of pass@\(k\): as \citet{slocum2025diverse} also observe, for one
fixed task the maximizing temperature is independent of \(k\).

Building on that fixed-task observation and the prior qualitative hard/easy-task
explanation, we give a formal population-level sufficient condition for the
aggregate pattern.
For task \(X\), let \(p_t(X)\) be one-sample success probability at temperature
\(t\), and define the conditional log-success response
\(m_t(u)=\E[\dot p_t(X)\mid p_t(X)=u]/u\). If \(m_t(u)\) is
nonincreasing in current success probability, then the
normalized temperature derivative of aggregate pass@\(k\) is nondecreasing in
\(k\). Consequently, derivative signs are nested across budgets; if each
temperature-performance curve is strictly single-peaked, its unique maximizer
is nondecreasing in \(k\). The proof identifies the mechanism as a
monotone-likelihood-ratio power tilt toward lower-success tasks.

We derive a closed-form two-stratum phase diagram, including upward and
downward regimes, and show that the marginal temperature derivative admits an
exact \(\mathrm{Beta}(2,k)\) kernel representation whose kernel concentrates at
one-sample success of order \(1/k\). Interpreting that scale as task-level
localization additionally requires a regular, nonvanishing density--response
factor near zero. A signed-moment representation yields diagnostic shape
restrictions, while a short appendix records exact discrete refinements of the
existing multi-configuration allocation formulation. No language model is
trained, and no model query is used as an experimental measurement: the
contribution is a conditional theory of an established empirical phenomenon,
with assumptions that can be tested in future work.
\end{abstract}

\section{Introduction}

Sampling temperature mediates a familiar quality--exploration trade-off in
language-model inference. For pass@\(k\), only one accepted completion among
\(k\) attempts is needed. Empirically, the temperature that performs best
often increases with the number of attempts. The Codex study explicitly
reported an optimum near \(0.2\) for pass@1 and \(0.8\) for pass@100 in one
model, and plotted a rising upper hull across budgets
\citep{chen2021codex}. Later studies reproduced the qualitative small-budget
low-temperature and large-budget high-temperature pattern across models and
tasks \citep{du2025temperature,chow2025inference}. Recent work also finds that
different temperatures solve different subsets of problems and that a
temperature portfolio can outperform a single temperature
\citep{wu2025role}.

The standard explanation is that a larger budget rewards diversity. That
intuition is useful but incomplete under the conditional-independent model
used to define population pass@\(k\). For a fixed task, pass@\(k\) is a
strictly increasing transformation of its one-sample success probability.
Therefore, the temperature maximizing a fixed task cannot move with \(k\).
\citet{slocum2025diverse} state this observation explicitly and give the
adjacent intuition that low-success tasks receive larger marginal gains at high
budgets. Movement of the benchmark-level optimum requires heterogeneous
temperature responses across tasks, or a violation of conditional i.i.d.
sampling, fixed-within-schedule decoding, or perfect verification.

This paper asks a narrow question:
\begin{quote}
Under what task-level response condition must the aggregate optimal
temperature move weakly upward as the pass@\(k\) budget grows?
\end{quote}
The answer uses a conditional log-success response. At temperature \(t\),
tasks are ordered by current one-sample success \(p_t\). If increasing
temperature has a weakly larger proportional benefit on lower-success tasks,
then the pass@\(k\) marginal shifts monotonically in favor of higher
temperature as \(k\) grows. The shift is not in the raw derivative magnitude;
it is in a normalized derivative whose weighting law is ordered by monotone
likelihood ratio. This distinction prevents a common but false claim that
\(\partial_t\pass@k\) itself must grow with \(k\).

Against that prior background, our contributions are:
\begin{enumerate}
    \item a sufficient-condition theorem giving nested temperature-derivative
    signs for every pair of budgets and, under strict single-peakedness,
    nondecreasing optimal temperatures;
    \item an exact two-stratum affine model with a phase boundary, closed-form
    optimizer, upward and downward regimes, and an explicit large-budget
    limit;
    \item a \(\mathrm{Beta}(2,k)\) kernel representation and its conditional
    localization interpretation, a logit-level interpretation of the
    response-order condition, and signed-moment diagnostics for future
    empirical tests.
\end{enumerate}

The mathematical tools---likelihood-ratio order, covariance inequalities,
monotone comparative statics, and the compact moment problem---are classical
\citep{lehmann1955ordered,karlin1956decision,milgrom1994monotone,
hildebrandt1933moment}. The contribution is their specialization to
inference-time temperature under pass@\(k\), not a new theorem about those
tools. We also do not claim to discover hard-task reweighting: the factor
\(k(1-p)^{k-1}\) and its implications for pass@\(k\) optimization already
appear in recent work \citep{barakat2026passk}. Likewise, optimized mixtures of
temperatures and other configurations are the subject of OSCA
\citep{zhang2025osca}. Our allocation appendix is explicitly a discrete
refinement of that framework.

\section{Population pass@k}
\label{sec:setup}

Let \(X\sim\mu\) denote a benchmark task, let \(I=[\underline t,\overline t]\)
be a compact temperature interval, and let
\[
p_t(X)=\Pr\{\text{one completion is accepted}\mid X,t\}.
\]
Conditional on \((X,t)\), the \(k\) completions are independent and identically
distributed, and the verifier is perfect for the binary success event. Write
\(q_t=1-p_t\). Population pass@\(k\) is
\begin{equation}
A_k(t)=\E\!\left[1-q_t(X)^k\right].
\label{eq:Ak}
\end{equation}
This is the population target estimated by the usual unbiased finite-sample
pass@\(k\) estimator \citep{chen2021codex}.

For the differential results, fix an interior operating point \(t\). Assume
that \(s\mapsto p_s(X)\) is differentiable almost surely in a neighborhood of
\(t\), that \(\E|\dot p_t(X)|<\infty\), that differentiation may pass through
the expectation, and that \(0<p_t(X)<1\) almost surely. We state the
differential theorems only at such points. Boundary cases require separate
one-sided derivatives and dominated-limit conditions. We use a dot for
\(\partial_t\).

\subsection{Background: fixed-task budget invariance}

The following elementary fact was stated explicitly by
\citet[Appendix~C]{slocum2025diverse}. We include it as background because it
separates a per-task identity from the population comparative static developed
below.

\begin{proposition}[Background budget invariance for a fixed task]
\label{prop:nogo}
For every fixed task \(x\) and every integer \(k\geq1\),
\[
\argmax_{t\in I}\{1-(1-p_t(x))^k\}
=\argmax_{t\in I}p_t(x).
\]
More generally, if two temperatures \(s,t\) satisfy
\(p_s(X)\geq p_t(X)\) almost surely, then
\(A_k(s)\geq A_k(t)\) for every \(k\).
\end{proposition}

\begin{proof}
For \(k\geq1\), the map \(u\mapsto1-(1-u)^k\) is strictly increasing on
\([0,1]\). Apply it pointwise, and then take expectations for the second
claim.
\end{proof}

Thus a changing aggregate optimum cannot be obtained by saying only that
larger \(k\) rewards repeated draws. Under the model, a fixed task's output
diversity matters for pass@\(k\) only through the total probability mass of
accepted outputs, \(p_t(x)\). Budget dependence requires temperature rankings
to cross across tasks.

\section{A monotone-temperature theorem}
\label{sec:main}

Differentiating \eqref{eq:Ak} gives the known pass@\(k\) weighting identity
\begin{equation}
A_k'(t)=k\,\E\!\left[q_t^{k-1}\dot p_t\right].
\label{eq:derivative}
\end{equation}
Define the pointwise log-success temperature response (a semi-elasticity)
\[
\eta_t(X)=\partial_t\log p_t(X)=\frac{\dot p_t(X)}{p_t(X)}
\]
and let \(P=p_t(X)\). For \(u>0\), define the conditional log-success
response
\begin{equation}
m_t(u)=\frac{\E[\dot p_t(X)\mid P=u]}{u}.
\label{eq:mt}
\end{equation}
This equals \(\E[\eta_t(X)\mid P=u]\) whenever \(\eta_t\) is integrable,
while the displayed definition is well defined under the stated integrability
of \(\dot p_t\).
Set
\[
Z_{k,t}=\E[P(1-P)^{k-1}]
\]
and, when \(Z_{k,t}>0\), define a probability law \(\nu_{k,t}\) on tasks by
\begin{equation}
\frac{\dd\nu_{k,t}}{\dd\mu}(X)
=\frac{p_t(X)q_t(X)^{k-1}}{Z_{k,t}}.
\label{eq:nu}
\end{equation}
Equations \eqref{eq:derivative}--\eqref{eq:nu} imply
\begin{equation}
A_k'(t)=kZ_{k,t}\,\E_{\nu_{k,t}}[m_t(P)].
\label{eq:normalized}
\end{equation}

\begin{assumption}[Decreasing conditional log-success response]
\label{ass:response}
At every interior \(t\) to which the standing differential assumptions
apply, a version of \(m_t(u)\) is nonincreasing in \(u\) on the support of
\(p_t(X)\).
\end{assumption}

This is a hard-task-benefit condition: at the current operating point, raising
temperature has a larger proportional response, or a smaller proportional
cost, on lower-success tasks. It is a restriction on observable task-level
responses, not a consequence of pass@\(k\). The conventional elasticity with
respect to temperature is \(t\,m_t(u)\); at any fixed positive \(t\), this
factor does not change the cross-task order assumed here.

\begin{lemma}[Budget induces a likelihood-ratio tilt]
\label{lem:tilt}
For integers \(\ell>k\),
\[
\frac{\dd\nu_{\ell,t}}{\dd\nu_{k,t}}(P)
=\frac{(1-P)^{\ell-k}}
       {\E_{\nu_{k,t}}[(1-P)^{\ell-k}]}.
\]
The likelihood ratio is nonincreasing in \(P\); hence
\(\nu_{\ell,t}\) is shifted toward lower one-sample success in monotone
likelihood-ratio order.
\end{lemma}

\begin{proof}
Divide the two densities in \eqref{eq:nu}. The remaining factor is
\((1-P)^{\ell-k}\), and the denominator normalizes it.
\end{proof}

\begin{theorem}[Nested temperature-derivative signs]
\label{thm:nested}
At any interior \(t\) satisfying the standing differential assumptions and
Assumption~\ref{ass:response}, for integers \(\ell>k\geq1\),
\begin{equation}
\E_{\nu_{\ell,t}}[m_t(P)]
\geq
\E_{\nu_{k,t}}[m_t(P)].
\label{eq:score-order}
\end{equation}
Consequently,
\begin{equation}
A_k'(t)\geq0\quad\Longrightarrow\quad A_\ell'(t)\geq0.
\label{eq:sign-nesting}
\end{equation}
For adjacent budgets, the exact increment of the normalized score is
\begin{equation}
\E_{\nu_{k+1,t}}m_t-\E_{\nu_{k,t}}m_t
=
\frac{\operatorname{Cov}_{\nu_{k,t}}
      (m_t(P),1-P)}
     {\E_{\nu_{k,t}}[1-P]}
\geq0.
\label{eq:covariance}
\end{equation}
\end{theorem}

\begin{proof}
By Lemma~\ref{lem:tilt},
\[
\E_{\nu_{\ell,t}}m_t
=
\frac{\E_{\nu_{k,t}}
[m_t(P)(1-P)^{\ell-k}]}
{\E_{\nu_{k,t}}[(1-P)^{\ell-k}]}.
\]
Both \(m_t(P)\) and \((1-P)^{\ell-k}\) are nonincreasing functions of
\(P\). Their covariance is nonnegative by the elementary association
inequality for comonotone functions, which proves
\eqref{eq:score-order}. Equation \eqref{eq:normalized} then gives
\eqref{eq:sign-nesting}. Taking \(\ell=k+1\) yields
\eqref{eq:covariance}.
\end{proof}

\begin{remark}[What is and is not monotone]
The theorem orders the normalized derivative
\(A_k'(t)/(kZ_{k,t})\). It does not claim that the raw derivatives
\(A_k'(t)\) increase in magnitude with \(k\); the positive normalizer can
shrink rapidly.
\end{remark}

\begin{definition}[Strict derivative single-peakedness]
\label{def:singlepeak}
A differentiable function \(f\) on \(I\) is strictly single-peaked if it has
a unique maximizer \(t^\star\), \(f'(t)>0\) for
\(t<t^\star\), and \(f'(t)<0\) for \(t>t^\star\), using one-sided
derivatives at the boundary.
\end{definition}

\begin{corollary}[Budget-monotone optimal temperature]
\label{cor:argmax}
Suppose the standing differential assumptions and
Assumption~\ref{ass:response} hold at every interior \(t\in I\), and every
\(A_k\) is strictly single-peaked with unique maximizer \(t_k\). Then
\[
t_\ell\geq t_k\qquad\text{whenever }\ell>k.
\]
\end{corollary}

\begin{proof}
If \(t_k\) is interior, then \(A_k'(t_k)=0\), so
Theorem~\ref{thm:nested} gives \(A_\ell'(t_k)\geq0\).
Strict single-peakedness of \(A_\ell\) rules out \(t_\ell<t_k\).
If \(t_k=\underline t\), the conclusion is immediate. If
\(t_k=\overline t\), then \(A_k'(t)>0\) throughout the interior;
sign nesting gives \(A_\ell'(t)\geq0\) there, forcing
\(t_\ell=\overline t\).
\end{proof}

Neither Assumption~\ref{ass:response} nor single-peakedness is automatic.
Without them, the optimal temperature may decrease, oscillate, or be
nonunique. The result is therefore a falsifiable sufficient condition, not a
universal law.

\section{A Beta-kernel representation of the marginal decision}
\label{sec:localization}

Suppose \(P=p_t(X)\) has a density \(f_t\) on \((0,1)\). Conditioning
\eqref{eq:derivative} on \(P\) gives
\[
A_k'(t)
=k\int_0^1u(1-u)^{k-1}m_t(u)f_t(u)\dd u.
\]
The density of \(Z_k\sim\mathrm{Beta}(2,k)\) is
\(k(k+1)u(1-u)^{k-1}\), so
\begin{equation}
A_k'(t)=\frac{1}{k+1}
\E\!\left[m_t(Z_k)f_t(Z_k)\right].
\label{eq:beta}
\end{equation}
The kernel itself has mode \(1/k\) for \(k>1\), mean \(2/(k+2)\), and
concentrates on one-sample success probabilities of order \(1/k\). Equation
\eqref{eq:beta} averages the density--response factor
\(h_t(u)=m_t(u)f_t(u)\) against that kernel. Kernel concentration alone does
not imply that the nonzero task-level contribution is centered at the same
scale. Such an interpretation is justified, for example, when \(h_t\) is
continuous with a finite nonzero limit near zero. If \(h_t\) vanishes or varies
rapidly there, or if the success distribution is supported away from zero, the
nonzero contribution can come from a different scale. For a finite benchmark,
the density representation itself is only an approximation or analogy.

The same tilt underlies aggregate inference-scaling laws. A mixture of
taskwise exponential failure curves can exhibit much slower aggregate
decay when the success distribution has substantial mass near zero
\citep{schaeffer2025monkeys}. Equation \eqref{eq:beta} asks a different,
local question: which part of that distribution controls the response to a
change in temperature?

\section{An exact two-stratum phase diagram}
\label{sec:two}

Consider an easy stratum of mass \(\pi\in(0,1)\) and a hard stratum of
mass \(1-\pi\), with affine temperature responses
\begin{equation}
p_E(t)=e-at,\qquad p_H(t)=h+bt,
\label{eq:linear}
\end{equation}
where \(a,b>0\) and both probabilities remain in \((0,1)\) on \(I\).
Put \(A=1-e\), \(B=1-h\), and
\begin{equation}
C=\frac{\pi a}{(1-\pi)b}.
\label{eq:C}
\end{equation}
The aggregate failure probability is
\[
F_k(t)=1-A_k(t)
=\pi(A+at)^k+(1-\pi)(B-bt)^k.
\]

\begin{proposition}[Closed-form optimizer and phase boundary]
\label{prop:twotype}
For \(k=1\), a maximizer is \(\underline t\) when \(C>1\),
\(\overline t\) when \(C<1\), and every \(t\in I\) when \(C=1\).
For \(k>1\), \(A_k\) is strictly concave and has the unique maximizer
\begin{equation}
t_k=
\clip_I\!\left(
\frac{B-r_kA}{b+r_ka}
\right),
\qquad
r_k=C^{1/(k-1)}.
\label{eq:twotype-opt}
\end{equation}
If \(C>1\), the sequence \((t_k)\) is nondecreasing; if \(C<1\), it is
nonincreasing. If \(C=1\), it is constant for \(k>1\). In every case,
\begin{equation}
\lim_{k\to\infty}t_k
=\clip_I\!\left(\frac{B-A}{a+b}\right),
\label{eq:twotype-limit}
\end{equation}
the temperature at which the two strata have equal one-sample success,
clipped to \(I\).
\end{proposition}

\begin{proof}
For \(k=1\),
\[
A_1'(t)=(1-\pi)b-\pi a,
\]
which gives the boundary cases through \eqref{eq:C}. For \(k>1\),
\[
A_k''(t)
=-k(k-1)\!\left[
\pi a^2(A+at)^{k-2}
+(1-\pi)b^2(B-bt)^{k-2}
\right]<0.
\]
The unique unconstrained critical point solves
\[
\frac{B-bt}{A+at}
=\left(\frac{\pi a}{(1-\pi)b}\right)^{1/(k-1)}
=r_k,
\]
which yields \eqref{eq:twotype-opt}; concavity justifies clipping.
Writing \(T(r)=(B-rA)/(b+ra)\), we have
\[
T'(r)=-\frac{Ab+aB}{(b+ra)^2}<0.
\]
When \(C>1\), \(r_k\downarrow1\), so \(T(r_k)\) increases; when
\(C<1\), \(r_k\uparrow1\), so it decreases. Clipping preserves either
order. Taking \(r_k\to1\) proves \eqref{eq:twotype-limit}.
\end{proof}

The ratio \(C\) is a one-shot phase boundary. When \(C>1\), the
population-weighted easy-stratum loss from raising temperature dominates at
\(k=1\), so the optimum starts cold. As \(k\) increases, failures on the
hard stratum receive greater marginal weight and the optimum moves toward the
difficulty-crossing temperature. When \(C<1\), the direction reverses.

For the equal-mass family
\[
p_E(t)=0.60-0.30t,\qquad
p_H(t)=0.25+0.15t,\qquad t\in[0,1],
\]
\(C=2\), and \eqref{eq:twotype-opt} gives
\[
t_1=t_2=0,\quad
t_3\approx0.321,\quad
t_5\approx0.541,\quad
t_{10}\approx0.671,\quad
t_\infty=\frac79.
\]
The reverse family
\[
\pi=\frac15,\quad
p_E(t)=\frac7{10}-\frac35t,\quad
p_H(t)=\frac15+\frac35t,\quad t\in[0,1]
\]
has \(C=1/4\), keeps all probabilities strictly between zero and one, and
\[
t_1=1,\qquad t_2=\frac{29}{30},\qquad
t_3=\frac{13}{18},\qquad t_\infty=\frac5{12}.
\]
This explicit downward sequence refutes any unconditional claim that the
optimal temperature must rise with budget. The original easy/hard ordering
reverses at \(t=5/12\), so Assumption~\ref{ass:response} is not maintained
over the relevant region.

\begin{figure}[t]
    \centering
    \includegraphics[width=\textwidth]{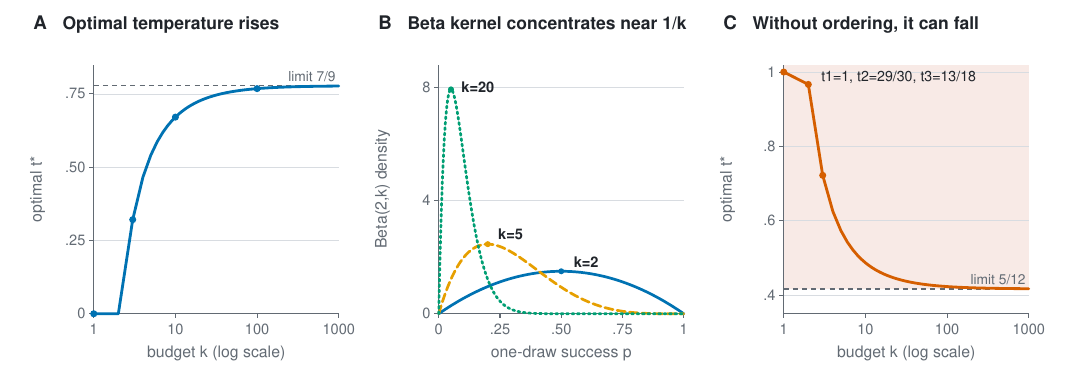}
    \caption{Analytical consequences of the theory. Left: the exact
    upward-regime optimizer from Proposition~\ref{prop:twotype}. Center:
    \(\mathrm{Beta}(2,k)\) kernels whose own mass shifts to success probabilities
    of order \(1/k\); task-level influence also depends on the density--response
    factor. Right: the valid downward-regime affine family, showing that upward
    movement is conditional rather than algebraic.}
    \label{fig:theory}
\end{figure}

\section{A completion-level interpretation}
\label{sec:micro}

Let \(\pi_t(y\mid x)\) be a probability mass function on a fixed countable
support, let \(C_x\) be the set of accepted completions, and assume that
\(t\mapsto\pi_t(\cdot\mid x)\) is differentiable at the operating point as a
map into \(\ell^1\). Also assume \(p_t(x)>0\) and
\(\pi_t(y\mid x)>0\) on the declared support. Summation over \(C_x\) is a
bounded linear functional on \(\ell^1\), so differentiation passes through the
accepted-set sum; the \(\ell^1\) derivative also makes the score below
integrable, both unconditionally and conditional on \(Y\in C_x\). Then
\[
p_t(x)=\sum_{y\in C_x}\pi_t(y\mid x)
\]
and the score identity gives
\begin{equation}
\partial_t\log p_t(x)
=
\E_{\pi_t}\!\left[
\partial_t\log\pi_t(Y\mid x)
\mid Y\in C_x
\right].
\label{eq:score-id}
\end{equation}
For the idealized global Gibbs family at \(t>0\), let the score \(s_x\) be
independent of \(t\) and write
\[
\pi_t(y\mid x)=\frac{\exp(s_x(y)/t)}{Z_x(t)},
\qquad
Z_x(t)=\sum_y\exp(s_x(y)/t).
\]
Assume that \(Z_x(t)\) is finite and differentiable in a neighborhood of the
operating point, that differentiation may pass through the defining sums, and
that \(\E_t|s_x(Y)|\) and
\(\E_t[|s_x(Y)|\mid Y\in C_x]\) are finite. Then
\begin{equation}
\partial_t\log p_t(x)
=
\frac{
\E_t[s_x(Y)]
-\E_t[s_x(Y)\mid Y\in C_x]
}{t^2}.
\label{eq:gibbs}
\end{equation}
Higher temperature therefore helps when accepted completions occupy
lower-score modes than the current global average, and hurts when accepted
completions already occupy the top modes.

Autoregressive tokenwise temperature has an analogous identity when the summed
token score is integrable and differentiation may pass through the path sum.
The derivative
of a completion log probability is the sum, over prefixes, of the
chosen-token logit minus the prefix-average logit, with the appropriate
\(-1/t^2\) sign. Averaging that score over accepted paths recovers
\eqref{eq:score-id}. Assumption~\ref{ass:response} therefore says, roughly,
that accepted paths for currently low-success tasks are deeper in the model's
score landscape than accepted paths for currently high-success tasks. This is
a mechanism that can generate the condition, not a proof that real models
always satisfy it.

\section{Diagnostics for future tests}
\label{sec:diagnostics}

Fix \(t\), assume \(\E|\dot p_t|<\infty\), and define a finite signed
measure on \([0,1]\) by
\[
\sigma_t(B)
=
\E\!\left[
\dot p_t(X)\,\1\{q_t(X)\in B\}
\right].
\]
Equation \eqref{eq:derivative} implies
\begin{equation}
g_j(t):=\frac{A_{j+1}'(t)}{j+1}
=\int_0^1q^j\dd\sigma_t(q),
\qquad j=0,1,\ldots.
\label{eq:moments}
\end{equation}

\begin{proposition}[Signed-moment identification and a shape test]
\label{prop:moments}
The infinite exact sequence \(\{g_j(t)\}_{j\geq0}\) uniquely determines
\(\sigma_t\). If \(\dot p_t(X)\geq0\) almost surely, then for all integers
\(j,r\geq0\),
\begin{equation}
(-1)^r\Delta^r g_j(t)
=
\E[q_t^j p_t^r\dot p_t]\geq0.
\label{eq:complete}
\end{equation}
\end{proposition}

\begin{proof}
If two finite signed measures on a compact interval have the same moments,
they integrate every polynomial equally. Polynomial density in \(C[0,1]\)
then implies equality of the measures. Repeated forward differences in
\eqref{eq:moments} multiply the integrand by \((q-1)^r=(-p)^r\),
which proves \eqref{eq:complete}.
\end{proof}

This is an application of the classical compact-interval moment problem
\citep{hildebrandt1933moment}, not a claim of a new moment theorem. For a
general signed response measure, complete monotonicity need not hold. Moreover,
identification uses the entire infinite, noiseless sequence; inversion from
finitely many noisy derivatives is ill-posed \citep{gerth2021hausdorff}. The
useful output is therefore a set
of diagnostics:
\begin{enumerate}
    \item estimate taskwise \(p_t\) and local log-success responses at nearby
    temperatures;
    \item test whether the conditional response \(m_t(u)\) is nonincreasing,
    preferably with held-out tasks and uncertainty bands;
    \item separately test strict single-peakedness of aggregate
    \(t\mapsto A_k(t)\);
    \item only when both conditions are supported, predict a nondecreasing
    sequence of optimal temperatures;
    \item use violations of \eqref{eq:complete} to reject the stronger
    hypothesis that the temperature increase helps every task.
\end{enumerate}
These are preregisterable predictions for future empirical work. They are not
evidence for the present theorem.

\section{Related work and claim boundaries}
\label{sec:related}

\paragraph{Temperature and multi-sample inference.}
\citet{chen2021codex} documented the increasing optimal-temperature pattern for
code generation. \citet{du2025temperature} studied temperature selection for
majority vote and best-of-\(N\), reporting lower temperatures at small samples,
higher temperatures at larger samples, and roughly single-peaked curves. They
also propose an entropy-based selector. \citet{chow2025inference} reported
temperature/sample co-scaling and an easy/hard-task interpretation.
\citet[Appendix~C]{slocum2025diverse} explicitly state both
fixed-task budget invariance and the qualitative asymmetry by which low-success
tasks can matter more at high budgets. Their variance interpretation does not
follow from concavity alone: a mean-preserving spread weakly lowers
the expectation of a concave function. Our response-order assumption instead
states the directional task heterogeneity needed for a rigorous ordering.
\citet{dang2025ensembling} formalize a bias--variance view of pass@\(k\) and
show that temperature scaling trades bias against variance. These works
establish the empirical target and neighboring mechanisms; our
contribution is the conditional derivative-sign and optimizer ordering, not
discovery of the pattern, fixed-task invariance, or its first qualitative
explanation.

\paragraph{Pass@\(k\) reweighting and inference scaling.}
\citet{barakat2026passk} derive the same
\(k(1-p)^{k-1}\) prompt weight in policy gradients and study how hard-prompt
emphasis can conflict with pass@1 optimization.
\citet{walder2025pkpo} directly optimize pass@\(k\), derive gradient
estimators, and report that larger target budgets prioritize harder problems.
\citet{schaeffer2025monkeys} use the distribution of one-attempt success to
explain aggregate power-law inference scaling. We specialize the
distributional view to a scalar inference-time control and ask when its
maximizers are ordered across budgets.

\paragraph{Configuration portfolios.}
\citet{zhang2025osca} formulate mixed allocation across configurations,
including temperature, model, and language; establish the convex relaxed
failure objective; and show that mixtures can beat one configuration.
\citet{wu2025role} empirically demonstrate complementary temperature-specific
task subsets. \citet{su2026adaptive} learn prompt- and budget-conditioned
decoding policies, a broader adaptive setting than the fixed-temperature
schedules studied here. Appendix~\ref{app:allocation} starts from the OSCA
objective and records exact integer endpoint and finite-type structural
consequences. We do not claim a new convex formulation or the first benefit of
mixing or budget-conditioned decoding.

\paragraph{Priority statement.}
\citet{slocum2025diverse} own the fixed-task invariance observation and a
qualitative hard/easy-task account of aggregate movement. To our knowledge, a
targeted primary-source search through July 19, 2026 did not locate a prior
result combining decreasing conditional log-success response with
derivative-sign nesting and ordered inference-time temperature maximizers. We
therefore make only that narrow sufficient-condition claim. The result is a
pass@\(k\)-temperature specialization of classical likelihood-ratio
comparative statics; it is not the first theory of pass@\(k\), the first
fixed-task no-go observation, the first hard-task explanation, a new general
monotone-comparative-statics theorem, or a universal law.

\section{Limitations and conclusion}

The analysis makes four substantive modeling commitments. Samples are
conditionally i.i.d.; the temperature is held fixed within each homogeneous
schedule; the verifier exactly recognizes success; and the benchmark average
is the population of interest. Correlation among attempts, adaptive decoding,
verifier error, or distribution shift can all change the objective. The
response order and single-peakedness must be checked rather than assumed.
Finite benchmarks can also produce grid and estimation noise that masks or
reverses the population ordering.

Within those boundaries, the result formalizes a precise sufficient condition
beyond the prior fixed-task observation. Pass@\(k\) does not make a fixed task
prefer a different temperature. It changes the benchmark-level marginal by a
power tilt toward tasks that still tend to fail.
When lower-success tasks have larger conditional log-success responses, that
tilt creates nested derivative signs; with single-peaked curves, the best
temperature can only move upward. The two-stratum phase diagram shows both why
the familiar pattern emerges and how it can reverse. This turns a broad
quality--diversity intuition into a testable conditional statement.

\section*{Tool-use disclosure}

OpenAI Codex was used to assist with literature search, algebraic and numerical
checks, figure generation, and manuscript drafting. It is not an author. The
named human author is responsible for independently verifying the claims,
references, and final submitted text and for taking responsibility for the
work.

\bibliographystyle{plainnat}
\bibliography{references}

\appendix

\section{Exact refinements of configuration allocation}
\label{app:allocation}

This appendix begins from the mixed-configuration failure objective studied by
OSCA \citep{zhang2025osca}. Its purpose is to make two discrete and asymptotic
consequences explicit, not to reclaim the general formulation.

\subsection{An exact integer criterion for two configurations}

Let \(q_A(X)\) and \(q_B(X)\) be per-sample failure probabilities for two
configurations, and write \(p_A=1-q_A\) and \(p_B=1-q_B\). Assume that,
conditional on \(X\), draws are independent across and within configurations.
With \(n\) samples from \(B\) and \(k-n\) from \(A\), define
\begin{equation}
F_k(n)=\E[q_A(X)^{k-n}q_B(X)^n],
\qquad n=0,\ldots,k.
\label{eq:Fn}
\end{equation}

\begin{proposition}[Discrete convexity and strict-mixing criterion]
\label{prop:mixing}
For \(n=0,\ldots,k-2\),
\begin{equation}
F_k(n+2)-2F_k(n+1)+F_k(n)
=
\E\!\left[
q_A^{k-n-2}q_B^n(q_A-q_B)^2
\right]\geq0.
\label{eq:discrete-convex}
\end{equation}
For \(k\geq2\), some interior integer schedule strictly beats both homogeneous
endpoints if and only if
\begin{align}
\E[q_A^{k-1}(p_A-p_B)]&<0,
\label{eq:left-gate}\\
\E[q_B^{k-1}(p_A-p_B)]&>0.
\label{eq:right-gate}
\end{align}
\end{proposition}

\begin{proof}
Expanding the second difference of \eqref{eq:Fn} gives
\eqref{eq:discrete-convex}. Hence the first differences form a
nondecreasing sequence. The left endpoint is not optimal exactly when
\[
F_k(1)-F_k(0)
=\E[q_A^{k-1}(q_B-q_A)]
=\E[q_A^{k-1}(p_A-p_B)]<0.
\]
Likewise, the right endpoint is not optimal exactly when
\[
F_k(k)-F_k(k-1)
=\E[q_B^{k-1}(p_A-p_B)]>0.
\]
For a convex sequence, both strict endpoint conditions hold if and only if
the minimum is attained strictly below both endpoints at an interior index.
\end{proof}

\subsection{Finite task types and the large-budget limit}

Suppose there are \(r\) task types with masses \(w_i>0\) summing to one, and
configurations \(j=1,\ldots,m\) have failure probabilities
\(q_{ji}\in(0,1]\).
Let \(a_j\in\R_+^r\) have coordinates
\(a_{ji}=-\log q_{ji}\). For relaxed allocation proportions
\(z\) in the simplex \(\Delta_m\),
\begin{equation}
F_k(z)
=\sum_{i=1}^r w_i
\exp\!\left(-k\sum_{j=1}^m z_ja_{ji}\right).
\label{eq:hazard-objective}
\end{equation}

\begin{proposition}[Sparse support, dominance, and maximin limit]
\label{prop:allocation-structure}
There exists an optimizer of \eqref{eq:hazard-objective} supported on at most
\(r\) configurations. A configuration whose hazard vector is componentwise
dominated by a convex combination of other hazard vectors can be removed
without worsening the optimum. Finally, with
\[
\phi_k(z)=-\frac1k\log F_k(z),
\qquad
m(z)=\min_i\sum_jz_ja_{ji},
\qquad
w_{\min}=\min_iw_i,
\]
the uniform bound
\begin{equation}
m(z)\leq\phi_k(z)
\leq m(z)+\frac{\log(1/w_{\min})}{k}
\label{eq:maximin-bound}
\end{equation}
holds. Every accumulation point of relaxed optimizers as \(k\to\infty\)
maximizes \(m(z)\).
\end{proposition}

\begin{proof}
Let \(h^\star=\sum_jz_j^\star a_j\) be an optimal point in the convex hull
of the hazard vectors, and let
\[
G(h)=\sum_iw_i e^{-kh_i}.
\]
First-order optimality implies
\(\nabla G(h^\star)\cdot(h-h^\star)\geq0\) for every feasible \(h\).
Thus \(h^\star\) lies in the exposed face maximizing
\(-\nabla G(h^\star)\cdot h\). The exposing normal is strictly positive and
nonzero, so the face has affine dimension at most \(r-1\). Carath\'eodory's
theorem within that face represents \(h^\star\) using at most \(r\) hazard
vectors.

If \(a_j\) is componentwise dominated by a convex combination of the other
vectors, replacing any mass on \(j\) by that combination weakly increases
every coordinate of \(h\). Since \(G\) is decreasing in every coordinate,
the objective cannot worsen.

Let \(i^\star\) attain \(m(z)\). Since the weights sum to one,
\[
w_{\min}e^{-km(z)}
\leq F_k(z)\leq e^{-km(z)}.
\]
Taking \(-k^{-1}\log\) gives \eqref{eq:maximin-bound}. The convergence
\(\phi_k\to m\) is uniform on the compact simplex, so every accumulation
point of maximizers of \(\phi_k\), equivalently minimizers of \(F_k\),
maximizes \(m\).
\end{proof}

\section{Additional proof details}

\subsection{Association inequality used in Theorem~\ref{thm:nested}}

If \(U,U'\) are independent with the same law and \(f,g\) are both
nonincreasing, then
\[
2\operatorname{Cov}(f(U),g(U))
=
\E[(f(U)-f(U'))(g(U)-g(U'))]\geq0.
\]
Apply this with \(f=m_t\) and
\(g(u)=(1-u)^{\ell-k}\).

\subsection{Boundary probability cases}

The log-success-response formulation is cleanest when \(0<p_t<1\). If a task has
\(p_t=0\) at an interior temperature and \(p_t\) is differentiable and
nonnegative in a neighborhood, then \(\dot p_t=0\), so the task does not
contribute to \eqref{eq:derivative} at that point. Tasks with \(p_t=1\)
receive zero weight for \(k>1\). Extensions can restrict \(\nu_{k,t}\) to
positive-weight tasks and take limits when \(Z_{k,t}>0\), one-sided
derivatives exist, and dominated convergence justifies the limits.

\section{Reproducibility}

The accompanying standard-library Python checks verify the closed-form
optimizers, upward and downward regimes, discrete-convexity identities,
finite-difference diagnostics, and displayed numerical values. The figure is
generated from the analytical formulas. These computations are regression
checks; the general results rely on the proofs above.

\end{document}